\documentclass{article}
\usepackage{ijcai26}

\usepackage{bbm}
\usepackage{amssymb}
\usepackage{times}
\usepackage{soul}
\usepackage{url}
\usepackage[hidelinks]{hyperref}
\usepackage[utf8]{inputenc}
\usepackage[small]{caption}
\usepackage{graphicx}
\usepackage{amsmath}
\usepackage{amsthm}
\usepackage{booktabs}
\usepackage{algorithm}
\usepackage{algorithmic}
\usepackage[switch]{lineno}
\usepackage{xcolor}
\usepackage{def}
\usepackage{cleveref}

\newtheorem*{theorem*}{Theorem}

\newtheorem{lemma}{Lemma}
\theoremstyle{definition}

\newtheorem{theorem}{Theorem}

\title{Simplicity Suffices for Parameter Noise Injection in Stochastic Gradient Descent}

\author{
    Benjamin Leblanc
    \and
    Louis-Jacob Lebel
    \and
    Teddy Kana
    \and
    Richard Kamel
    \affiliations
    Université Laval
}

\begin{document}

\maketitle

\begin{abstract}
    Injecting noise into the optimization process is a well-established technique for improving the training and generalization of deep neural networks. Yet, despite the breadth of existing approaches, it remains unclear which design choices truly matter in practice. In this work, we investigate parameter noise injection for stochastic gradient descent, focusing on two key questions: how to efficiently pair each training example with its own perturbation in mini-batch training, and whether sophisticated noise parameterizations or multi-sample gradient averaging yield meaningful gains over simpler alternatives. To address the first question, we leverage a distributional identity for linear layers that allows per-example noise injection without breaking batched computation. To address the second, we systematically compare several diagonal Gaussian parameterizations against an isotropic baseline across varying noise levels on CIFAR100. Our results consistently show that simple, lightweight strategies — isotropic noise with a single perturbed forward pass per update step — recover most of the benefit of more complex schemes. These findings suggest that simplicity suffices for parameter noise injection, and that practitioners need not resort to elaborate perturbation designs to reap the optimization and generalization benefits of noisy SGD.
\end{abstract}

\section{Introduction}

Non-convex optimization lies at the heart of modern machine learning, where models with millions or even billions of parameters must be trained efficiently. Among optimization methods, stochastic gradient descent (SGD) and its variants remain the backbone of most contemporary learning systems. Their practical success relies on the ability to navigate highly complex loss landscapes and converge toward solutions that generalize well.

To improve this behavior, many works have proposed injecting random perturbations (\textit{noise}) into the optimization process. This idea can be interpreted as a smoothing mechanism and has been studied in optimization for decades \cite{Rastrigin,Matyas,Duchi}. In machine learning, noise injection is also known as Perturbed Gradient Descent \cite{Cui} or Stochastic Gradient Langevin Dynamics \cite{DBLP:conf/icml/WellingT11}. It has been shown to alleviate optimization difficulties by helping escape sharp minima and saddle points \cite{Mishchenko}, while also improving generalization performance \cite{DBLP:conf/icml/SmithED20,DBLP:conf/icml/OrvietoKPBL22}. Noise-based methods have also proven useful in differential privacy \cite{DBLP:conf/icde/DuanHYS25} and privacy preserving in general \cite{DBLP:conf/nips/AltschulerT22}, adversarial robustness \cite{DBLP:conf/icml/CohenRK19}, reinforcement learning \cite{DBLP:conf/iclr/PlappertHDSC0AA18}, and variational quantum optimization \cite{Liu}.

The literature on noisy optimization is broad and heterogeneous. For instance, noise may be injected directly into the parameters \cite{DBLP:conf/iclr/ChaudhariCSLBBC17,Haruki,DBLP:conf/icml/OrvietoKPBL22}, into gradients \cite{DBLP:journals/corr/NeelakantanVLSK15}, or through gradient-free perturbation schemes \cite{Nesterov,DBLP:conf/icml/ReifensteinLY24}. Different perturbation distributions have been considered, most commonly Gaussian noise, but also Laplacian variants \cite{Osher}. Recent work has further adapted noise injection to common learning algorithms \cite{Starnes1,Starnes2,Starnes3}, such as gradient descent, stochastic gradient descent (SGD), and Adam \cite{DBLP:journals/corr/KingmaB14}, and modern training settings such as LoRA fine-tuning \cite{Chang}.

Despite this diversity of approaches, it remains unclear which design choices are truly important in practice. In particular, when injecting Gaussian parameter noise, does performance depend on sophisticated parameterizations and repeated perturbation averaging, or can simpler strategies suffice? In this paper, we investigate this question empirically. First, we introduce an efficient noise-injection method for linear layers in deep neural networks that enables one perturbation draw per training example, even under mini-batch training. Second, we compare several parameterizations of Gaussian weight perturbations. Third, we evaluate the benefit of averaging multiple perturbed gradients within a single update step.

Our results suggest that simplicity is often sufficient: lightweight parameterizations and a single noisy forward pass per optimization step already capture most of the benefits of parameter noise injection.

%
%
%
%
%

\section{Background and notation}

Each task is characterized by a dataset $S=\{\xbf_i,y_i\}_{i=1}^{n}$ containing $n$ instances. Each instance contains a set of features $\xbf\in\Xcal\subseteq\Rbb^m$ and a label $y\in\Ycal$. We consider that each instance $(\xbf,y)$ is the result of an unknown probability distribution $D$ defined over $\Xcal\times\Ycal$. We consider a predictor with fixed architecture $f_{\thetabf}~:~\Xcal\rightarrow\Ycal$ defined by a set of real-valued parameters $\thetabf$. Throughout this work, we may consider the input of $f$ to be a matrix $\Xbf$, a set of feature sets, so that $f$ acts on each feature set element-wise. A loss function $\ell~:~\Ycal\times\Ycal\rightarrow\Rbb_{\geq0}$ is used to quantify the quality of the predictions of $f_\thetabf$, and the training loss $\Lcal_S(\thetabf) = \frac{1}{n}\sum_{(\xbf,y)\in S} \ell(f_{\thetabf}(\xbf),y)$ serves to quantify the quality of the predictions of $f_\thetabf$ on $S$. Our goal is to leverage the training set $S$ to find a parameterization $\thetabf$ such that the generalization loss of $f_\thetabf$ (that is, its average loss over examples from $D$, $\Lcal_{D}(\thetabf) = \Esp_{(\xbf,y)\sim D}\left[\ell(f_{\thetabf}(\xbf),y)\right]$) is small.

\section{The Noise Injection Paradigm}

Let $Z$ be a continuous $m$-variate random variable defined over $\Rbb^{m}$ with probability density function $p_Z$. Perturbed parameters consist of their translation with respect to some random noise $\Zbf\sim Z$; i.e., $\thetabf$ becomes $\thetabf-\Zbf$. In the noise injection paradigm, instead of optimizing the training loss, we optimize the average training loss over perturbed parameters: $\Lcal_S^Z(\thetabf) = \Esp_{\Zbf\sim Z}\left[\Lcal_S(\thetabf-\Zbf)\right]$. While it has been observed that $\Lcal_S^Z(\thetabf)$ is smoother than $\Lcal_S(\thetabf)$, there is a mathematical reason for this: as it turns out, $\Lcal_S^Z(\thetabf)$ corresponds to the convolution between the original loss landscape $\Lcal_S(\thetabf)$ and the probability density function $p_Z$ of the random variable $Z$:
$$
\Lcal_S^Z(\thetabf) = \int_{\Rbb^m}\Lcal_S(\thetabf-\Zbf)p_Z(\Zbf) d\Zbf = (\Lcal_S~\star~p_Z)(\thetabf),
$$

where $\star$ is the convolution operator, an operator well-known for its smoothing properties.

Unfortunately, while in theory it is possible to compute the value of $\Lcal_S^Z(\thetabf)$ for any set of parameters, in practice, it is rarely the case. Modern deep neural architectures make the exact computation intractable. An alternative is found in empirically estimating this quantity and its gradient with respect to $\thetabf$ through $s\in\Nbb_{>0}$ random samples $\Zcal=\{\Zbf_j\}_{j=1}^s\sim Z^s$:
\begin{align*}
\Lcal_S^\Zcal(\thetabf) &= \frac{1}{s}\sum_{j=1}^{s} \Lcal_S(\thetabf-\Zbf_j),\\
\nabla_{\thetabf}\Lcal_S^\Zcal(\thetabf) &= \frac{1}{s}\sum_{j=1}^{s} \nabla_{\thetabf}\Lcal_S(\thetabf-\Zbf_j).
\end{align*}

\begin{algorithm}[tb]
    \caption{\textcolor{red}{Noisy} Stochastic Gradient Descent (\textcolor{red}{N}SGD)}
    \label{alg:NSGD}
    \textbf{Input}: Initial parameters $\thetabf^{(0)}$, training set $S=\{\xbf_i,y_i\}_{i=1}^{n}$,
    
    \quad\quad\quad learning rate $\eta$, batch size $B$, number of steps $T$, 

    \quad\quad\quad \textcolor{red}{noise distribution $Z$, number of noise samples $s$}.\\
    \textbf{Output}: Trained parameter vector $\thetabf^{(T)}$
    
    \begin{algorithmic}[1] 
        \FOR{$t = 0,\dots,T-1$}
        \STATE $\ibf=[i_1,\dots,i_B] \sim U(1,\dots,n)$ \COMMENT{Sample batch of training data through a uniform random index vector}
        \STATE \textcolor{red}{$\{\Zbf_j\}_{j=1}^s\sim Z^s$ ~\COMMENT{Random perturbation draw}}
        \STATE $\gbf_t = \textcolor{red}{\frac{1}{s}\sum_{j=1}^{s}}\nabla_{\thetabf} \Lcal_{S_\ibf}({\thetabf^{(t)}}\textcolor{red}{-~\Zbf_j})$
        \STATE $\thetabf^{(t+1)}\leftarrow \thetabf^{(t)}-\eta \gbf_t$
        \ENDFOR
        \STATE \textbf{return} $\thetabf^{(T)}$
    \end{algorithmic}
\end{algorithm}

\Cref{alg:NSGD} presents the modifications, through every red-colored character, made to the popular Stochastic Gradient Descent algorithm to create a noisy version of it.

\section{Efficient Gaussian Noise Injection in Stochastic Gradient Descent}\label{sec:4}

Since the computation of both $\Lcal_S^\Zcal(\thetabf)$ and $\nabla_{\thetabf}\Lcal_S^\Zcal(\thetabf)$ is proportional to the number $s$ of sampled noise $\Zbf_j$, each forward pass should be the most efficient possible. That is, it should explore the space $\Lcal_S^Z(\thetabf)$ as efficiently as possible. As depicted in \Cref{alg:NSGD}, the most common approach to noise injection in the stochastic gradient descent (SGD) algorithm is to apply (possibly many times) a random weight perturbation to a whole batch of examples simultaneously. Ideally, each example in a batch would be paired with its own random perturbation so that the noise space is more effectively explored. 

At first glance, this seems incompatible with many predictors family, such as neural networks. Indeed, let \mbox{$f(\cdot) = (L_1\circ\dots\circ L_l)(\cdot)$} be a feedforward neural network containing $l$ layers. Indeed, a linear layer $L_k~:~\Rbb^{d'}\rightarrow\Rbb^{d}$ containing $d$ hidden neurons is defined as
$$
L_k(\Xbf) = \sigma\left(\Xbf\Wbf_k\right) = \begin{bmatrix}
\sigma\left(\xbf_1\Wbf_k\right)\\
\vdots\\
\sigma\left(\xbf_n\Wbf_k\right)
\end{bmatrix}_{n\times d},
$$
with an activation function $\sigma$ acting element-wise and a weights matrix $\Wbf_k$\footnote{For simplicity, biases, without loss of generality, are omitted.} of size $d'\times d$. Let us focus on the perturbation of the weights $\Wbf_k$ and thus define $Z$ with according dimensions. Concerning the basic noisy counterpart of $L_k(\Xbf)$, that is,
$$
L_k^{Z}(\Xbf) = \sigma\left(\Xbf(\Wbf_k+Z)\right) = \begin{bmatrix}
\sigma\left(\xbf_1(\Wbf_k+Z)\right)\\
\vdots\\
\sigma\left(\xbf_n(\Wbf_k+Z)\right)
\end{bmatrix}_{n\times d},
$$
it is explicit that once parameter noise $\Zbf\sim Z$ is sampled, the same noise is applied to every example $\xbf_1,\dots,\xbf_n$ through the matrix product $\Xbf(\Wbf_k+\Zbf)$. The question becomes: how can one sample unique perturbations for every sample in the batch, while preserving the computational efficiency of a single matrix product in the feedforward layer output computation?

Our solution lies in injecting the noise not to the weights $\Wbf_k$, but to the resulting output of the matrix product $\Xbf\Wbf_k$. To do so, we first fix the noise distribution: an isotropic Gaussian multivariate distribution (that is, \mbox{$Z=\mathcal{N}_{d'\times d}(\mathbf{0},\sigma^2\cdot I)$}); this choice is based on the simplicity of the distribution and is used in most works surveyed in the introduction. Then, we need to compute the resulting data distribution of $\Xbf\Wbf_k$ by leveraging the following lemma.

\begin{lemma}\label{lem:gaussian_variation}
    Let $\mathbf{x}\in\mathbb{R}^{1\times d'}$ and $Z=\mathcal{N}_{d'\times d}(\mathbf{0},\sigma^2\cdot I)$, where $I$ is the identity matrix. Then:
    $$
    \mathbf{x}Z=\mathcal{N}_d(\mathbf{0},\sigma^2||\mathbf{x}||^2\cdot I).
    $$
\end{lemma}

This leads to the following result, permitting the injection of unique noise to every example of a batch in the forward pass of a feedforward neural network without sacrificing any computational efficiency.

\begin{theorem}\label{theo:efficient_nsgd}
Let $Z^n = Z_1\times\dots\times Z_n$ be a set of \textit{iid} random variables $Z=\mathcal{N}_{d'\times d}\left(\mathbf{0},\sigma^2\cdot I\right)$. Denoting the random variable
$$
L_k^{Z^n}(\Xbf) = \sigma\left(\Xbf(\Wbf_k+Z^n)\right) = \begin{bmatrix}
\sigma\left(\xbf_1(\Wbf_k+Z_1)\right)\\
\vdots\\
\sigma\left(\xbf_n(\Wbf_k+Z_n)\right)
\end{bmatrix}_{n\times d},
$$
we have
$$
L_k^{Z^n}(\Xbf) = \sigma(\Xbf\Wbf_k+\boldsymbol\psi),
$$
where $\psi_i=\mathcal{N}_d\left(\mathbf{0},\sigma^2||\mathbf{x}_i||^2\cdot I\right)$.
\end{theorem}

\begin{proof}
    The result follows from the observation that 
\begin{align*}
L_k^{Z^n}(\Xbf) &= \begin{bmatrix}
\sigma\left(\xbf_1(\Wbf_k+Z_1)\right)\\
\vdots\\
\sigma\left(\xbf_n(\Wbf_k+Z_n)\right)
\end{bmatrix}_{n\times d}\\
&= \begin{bmatrix}
\sigma\left(\xbf_1\Wbf_k+\xbf_1Z_1\right)\\
\vdots\\
\sigma\left(\xbf_n\Wbf_k+\xbf_nZ_n\right)
\end{bmatrix}_{n\times d}
\end{align*}
and an application of \Cref{lem:gaussian_variation} to every $\xbf_iZ_i$.
\end{proof}

Using this procedure to efficiently assign each example its own perturbation in the SGD algorithm is described in \Cref{alg:ENSGD}. Indeed, in \Cref{alg:ENSGD}, we consider the generalization of \Cref{theo:efficient_nsgd} for both weights and biases, which simply consists in defining $\psi_i=\mathcal{N}_d\left(\mathbf{0},\sigma^2(||\mathbf{x}_i||^2+1)\cdot I\right)$.

\begin{algorithm}[tb]
    \caption{Efficient Noisy Stochastic Gradient Descent (ENSGD) for fully connected neural networks}
    \label{alg:ENSGD}
    \textbf{Input}: Parameters $\thetabf^{(0)}=\{(\Wbf_k^{(0)},\bbf_k^{(0)})\}_{k=1}^{l}$, training set

    \quad\quad\quad $S=\{\xbf_i,y_i\}_{i=1}^{n}$, learning rate $\eta$, batch size $B$,

    \quad\quad\quad number of steps $T$, noise distribution $Z$, number

    \quad\quad\quad of noise samples $s$.\\
    \textbf{Output}: Trained parameter vector $\thetabf^{(T)}$

    \begin{algorithmic}[1] 
        \FOR{$t = 0,\dots,T-1$}
        \STATE $\ibf=[i_1,\dots,i_B] \sim U(1,\dots,n)$ \COMMENT{Sample batch of training data through a uniform random index vector}
        \STATE $\Xbf^{(0,j)} \leftarrow \Xbf_{\ibf}~\forall j\in\{1,\dots,s\}$
        \FOR{$k = 1,\dots,l$}
        \STATE $\zbf_{j,i}\sim Z_i',~\forall i\in[i_1,\dots,i_B],j\in\{1,\dots,s\}$ ~\COMMENT{Random perturbation draw, where \mbox{$Z_i' = \Ncal_{d'}(\mathbf{0},\sigma^2(||\xbf_i^{(k-1,j)}||^2+1)\cdot I)$}}
        \STATE $\Xbf^{(k,j)}\hspace{-1mm}\leftarrow\hspace{-1mm}\sigma(\Wbf_k^{(t)}\Xbf^{(k-1,j)}\hspace{-0.85mm}+\hspace{-0.85mm}\bbf_k^{(t)}\hspace{-0.85mm}+\hspace{-0.85mm}\Zbf_j),\forall j\hspace{-0.75mm}\in\hspace{-0.75mm}\{1,\dots,s\}$
        \ENDFOR
        \STATE $\gbf_t = \frac{1}{s}\sum_{j=1}^{s}\nabla_{\thetabf} \Lcal_{S_\ibf}({\thetabf^{(t)}}-~\Zbf_j)$~\\\COMMENT{where $f_{\thetabf^{(t)}-\Zbf_j}(\Xbf_{\ibf}) = \Xbf^{(l,j)}$}
        \STATE $\thetabf^{(t+1)}\leftarrow \thetabf^{(t)}-\eta \gbf_t$
        \ENDFOR
        \STATE \textbf{return} $\thetabf^{(T)}$
    \end{algorithmic}
\end{algorithm}

We emphasize that the result from \Cref{theo:efficient_nsgd} and the application depicted in \Cref{alg:ENSGD} can easily be generalized to other usages. For instance, the noise injection method can be used on any linear layer of any type of neural network. Also, we considered isotropic Gaussian distributions for simplicity, yet it is easy to generalize \Cref{theo:efficient_nsgd} to multivariate Gaussian noise distributions with a diagonal covariance matrix.

\section{Experiments}

\paragraph{Assessing the relevance of the ENSGD algorithm.}

We first empirically verify whether our Efficient Noisy Stochastic Gradient Descent algorithm (ENSGD, \Cref{alg:ENSGD}) yields better performances than the vanilla approach (NSGD, \Cref{alg:NSGD}). To do so, we train a simple convolutional neural network (two convolutional layers and three linear layers of width 120) on the CIFAR100 \cite{Krizhevsky09learningmultiple} task. As discussed in \Cref{sec:4}, we consider an isotropic Gaussian multivariate distribution for the noise distribution, with various values for the $\sigma$ parameter, log-uniformly spaced between 0.00125 and 0.08. We report the test 0-1 loss obtained over 5 random initializations, using the Adam base algorithm, for a maximum of 128 epochs, with a learning rate of $5\cdot10^{-4}$.

\begin{table*}
    \centering
    \begin{tabular}{|c||c|c|c|c|c|c|c|}
        \hline
        \textbf{Noise injection method $\backslash$ $\sigma$ value} & \textbf{0.00125} & \textbf{0.0025} & \textbf{0.005} & \textbf{0.01} & \textbf{0.02} & \textbf{0.04} & \textbf{0.08} \\
        \hline
        \hline
        \textbf{NSGD (\Cref{alg:NSGD})} & 0.6909 & 0.6834 & 0.6698 & 0.6575 & 0.6540 & \textbf{0.6508} & 0.7208 \\
        \hline
        \textbf{ENSGD (\Cref{alg:ENSGD})} & 0.6909 & 0.6796 & 0.6623 & 0.6356 & \textbf{0.6241} & 0.6258 & 0.6759 \\
        \hline
    \end{tabular}
    \caption{Best mean test 0-1 loss over 5 random seeds, with varying noise values, for both the NSGD and the ENSGD algorithms. \\ \textbf{Bold}: overall best test error per algorithm.}
    \label{tab:first}
\end{table*}


\begin{table*}
    \centering
    \begin{tabular}{|c||c|c|c|c|c|c|c|c|c|c|}
        \hline
        \textbf{Noise injection method $\backslash$ Num. of noise samples $s$} & \textbf{1} & \textbf{2} & \textbf{4} & \textbf{8} & \textbf{32} & \textbf{64} & \textbf{128}\\
        \hline
        \hline
        \textbf{ENSGD (\Cref{alg:ENSGD})} & \textbf{0.6185} & 0.6198 & 0.6222 & 0.6233 & 0.6276 & 0.6252 & 0.6261\\
        \hline
    \end{tabular}
    \caption{Best mean test 0-1 loss over 5 random seeds, with varying noise values, for the ENSGD algorithm. \textbf{Bold}: overall best test error.}
    \label{tab:second}
\end{table*}

\begin{table}
    \centering
    \small
    \begin{tabular}{lccc}
        \toprule
        \textbf{Noise distribution} & \textbf{CIFAR-100} & \textbf{AG News} & \textbf{IMDB} \\
        \midrule
        None (baseline) & 0.6985 & 0.0843 & 0.1110 \\
        \midrule
        Isotropic  & \underline{0.6241} & \underline{0.0837} & 0.1104 \\
        Moving     & 0.6256 & \textbf{0.0832}   & \textbf{0.1099} \\
        SqGrad     & \textbf{0.6213} & 0.0837            & 0.1122           \\
        InvSqGrad  & 0.6456 & 0.0851            & \underline{0.1103}           \\
        \bottomrule
    \end{tabular}
    \caption{Best mean test 0-1 loss over 5 random seeds, with varying noise values, for various noise distributions. \textbf{Bold}: overall best test error per method; \underline{underline}: runner-up.}
    \label{tab:third}
\end{table}

The results, reported in \cref{tab:first}, indicate that our ENSGD is relevant when compared to the vanilla noise injection approach. Indeed, it leads to better test error for every $\sigma$ value tested. 

\paragraph{Questioning the importance of the number of noise samples.}

It is natural to assume that the number of noise samples, the parameter $s$ in \Cref{alg:ENSGD}, must be big to obtain an adequate empirical estimation of $\Lcal_S^Z(\thetabf) = \Esp_{\Zbf\sim Z}\left[\Lcal_S(\thetabf-\Zbf)\right]$ during the training. To verify this, we consider the same empirical setting as earlier, but now varying the number of perturbation draws per weight update step. We treat the $\sigma$ values as a hyperparameter to be optimized with the help of a validation set. The results of this experiment are contained in \Cref{tab:second}.

We observe that increasing $s$ beyond a single draw (1) provides diminishing returns in terms of final test loss while significantly increasing the computational cost per step. This empirical evidence further reinforces our central claim: lightweight, simple noise injection strategies (using a single perturbation draw per example) capture the vast majority of optimization and generalization benefits.

\paragraph{Noise distribution comparison.}


Leveraging our ENSGD approach, we now consider various parameterizations of the variance matrix $\Sigma$ for the multivariate Gaussian distributions that generate the noise injected into the model, assuming a diagonal covariance matrix defined by the vector $\sigmabf_t$, which may vary with the current epoch number $t$. The \textit{Isotropic} parameterization method, described in \cref{sec:4} and used in most works presented in the introduction, is defined to generate noise with a constant variance over time and over every model parameter: $\sigma^*_{t,i} = \sigma$; the \textit{Moving} method traces an exponential moving average of the empirical variance of the last 5 gradients values, which we denote $\mathbf{g}^*_t$: $\sigma^*_{t,i} = \sigma\cdot g^*_{t,i}$; the \textit{SqGrad} method scales the standard deviation of the noise by Adam's second moment and then normalizes it: $\sigma^*_{t,i} = \sigma\frac{\frac{1}{t}\sum_{t'=0}^t g_{t',i}^2}{\frac{1}{t}\sum_{t'=0}^t\sum_{i} g_{t',i}^2}$; the \textit{InvSqGrad} method scales the standard deviation of the noise by the inverse of Adam's second moment and then normalizes it: $\sigma^*_{t,i} = \sigma\frac{\frac{1}{t}\sum_{t'=0}^t\sum_{i} g_{t',i}^2}{\frac{1}{t}\sum_{t'=0}^t g_{t',i}^2}$. We also experiment with no noise at all to confirm the relevance of the noise-injection paradigm.

To compare all of these noise distribution parameterizations, we perform the same experiment as earlier, but using the ENSGD algorithm and the various parameterization methods. We experiment on CIFAR-100 and on two standard natural language processing benchmarks: AG News~\cite{zhang2015character}, a four-class news topic classification task with 120{,}000 training articles, and IMDB~\cite{DBLP:conf/acl/MaasDPHNP11}, a binary sentiment classification task containing 50{,}000 movie reviews. In both settings, documents are represented as TF-IDF vectors over a vocabulary of 10{,}000 features and classification is performed by a fully-connected MLP (multi-layer perceptron) trained with the Adam underlying algorithm at a learning rate of $5\times10^{-4}$. We still treat the $\sigma$ values as a hyperparameter to be optimized with the help of a validation set.

The results, compiled in \cref{tab:third}, show that the isotropic method, while not yielding the best mean test error, is only slightly worse than the best method. That is, the gains from using more sophisticated methods (with higher computational overhead) are marginal (from 0.6241 to 0.6213). This experiment also confirms that the noise injection paradigm yields impressive gains in the image classification setting, thus confirming that the noise injection paradigm helps obtain better results on trained neural networks.  The absolute gains over the no-noise baseline are modest on both text classification datasets, at most $0.11$ percentage points, which is substantially smaller than what was observed on CIFAR100. This is expected because TF-IDF representations aggregate token occurrences into a fixed-length frequency vector, which produces an optimization landscape that is already comparatively smooth, leaving less room for noise injection to improve convergence. Despite this, the results remain consistent with our central finding: isotropic noise matches or closely approximates the best-performing method on both tasks, while more sophisticated parameterizations provide no systematic advantage; \textit{SqGrad}, in particular, fails to improve over the baseline on IMDB at any tested noise level. Crucially, no method causes significant degradation at its optimal noise level. Taken together, these results support the conclusion that simplicity suffices for parameter noise injection across diverse task settings, including regimes where the overall benefit of noise is limited by the smoothness of the underlying optimization landscape.



\section{Conclusion}

We investigated parameter noise injection for stochastic gradient descent, introducing ENSGD, an efficient method that assigns each training example its own perturbation without sacrificing batched computation. Through systematic experiments on image and text classification benchmarks, we showed that isotropic noise with a single perturbation drawn per update step consistently matches or closely approximates more sophisticated parameterizations. The computational overhead of diagonal Gaussian schemes and multi-sample gradient averaging yields only marginal gains, if any. These findings suggest that simplicity suffices for parameter noise injection: practitioners can confidently adopt lightweight noise strategies to improve optimization and generalization without resorting to elaborate perturbation designs.

In the near future, we plan on broadening the experimental section with more datasets in order to consolidate our findings. We also plan on extending the noise injection method from linear layers to convolutional layers; that is, linear layers involving weight sharing.

\newpage
\bibliographystyle{named}
\bibliography{ijcai26}

\end{document}